\documentclass[10pt,twocolumn,letterpaper]{article}

\usepackage{cvm}
\usepackage{times}
\usepackage{epsfig}
\usepackage{graphicx}
\usepackage{amsmath}
\usepackage{amssymb}
\usepackage{enumitem}
\usepackage[dvipsnames]{xcolor}         

\usepackage[pagebackref=true,breaklinks=true,letterpaper=true,colorlinks,bookmarks=false]{hyperref}

\newcommand{\keywords}[1]{{\bf \emph{Keywords: #1}}}

\definecolor{lightyellow}{RGB}{255, 255, 180}

\definecolor{lightred}{RGB}{0,0,0}
\newcommand{\texthl}[1]{\textcolor{lightred}{#1}}

\cvmfinalcopy

\def\cvmPaperID{****} 

\ifcvmfinal\fi
\usepackage[dvipsnames]{xcolor}
\usepackage{colortbl}
\usepackage{amsthm, amsmath}
\usepackage{booktabs}
\usepackage{arydshln}
\usepackage{graphicx}
\usepackage{subcaption}
\usepackage[capitalize,noabbrev]{cleveref}
\usepackage{enumitem}

\usepackage{algorithm,algorithmic}

\theoremstyle{plain}

\newtheorem{theorem}{Theorem} 
\newtheorem{proposition}{Proposition}

\newtheorem{remark}{Remark}

\definecolor{mygray}{gray}{.9}
\definecolor{myblue}{HTML}{25537D}
\definecolor{myred}{HTML}{E85642}

\newcommand\inc[1]{\textcolor{myred}{$\uparrow$ #1}}

\makeatletter
\def\hlinew#1{%
\noalign{\ifnum0=`}\fi\hrule \@height #1 \futurelet
\reserved@a\@xhline}
\makeatother%

\setlist[itemize]{leftmargin=0.5cm} 
\begin{document}

\title{Learning from Imperfect Text Guidance: Robust Long-Tail Visual Recognition with High-Noise Labels}

\author{
Mengke Li \\
CSSE, Shenzhen University\\
Shenzhen, China\\
{\tt\small mengkeli@szu.edu.cn}
\and
Haiquan Ling \\
CSSE, Shenzhen University\\
Shenzhen, China\\
{\tt\small 2410815024@mails.szu.edu.cn}
\and
Yiqun Zhang\\
{\small SCST, Guangdong University of Technology}\\
Guangzhou, China\\
{\tt\small yqzhang@gdut.edu.cn}
\and
Yang Lu\\
{\small INFORMATICS, Xiamen University} \\
Xiamen, China\\
{\tt\small luyang@xmu.edu.cn}
\and
Hui Huang\thanks{Corresponding author}\\
CSSE, Shenzhen University\\
Shenzhen, China\\
{\tt\small hhzhiyan@gmail.com}
}


\maketitle

\begin{abstract}
Real-world data often exhibit long-tailed distributions with numerous noisy labels, substantially degrading the performance of deep models.
While prior research has made progress in addressing this combined challenge, it overlooks the severe label-image mismatch inherent to high-noise settings, thereby limiting their effectiveness.
Given that observed labels, though mismatched with images, still retain category information, we propose employing auxiliary text information from labels to address label-image inconsistencies in long-tailed noisy data. 
Specifically, we leverage the intrinsic cross-modal alignment in pre-trained visual-language models to correct the label-image inconsistencies. 
This supervisory signal, referred to as Weak Teacher Supervision (WTS), is unaffected by label noise and data distribution biases, albeit exhibits limited accuracy.
Therefore, the activation of WTS is determined by evaluating the discrepancy between text-predicted labels and observed labels. 
Extensive experiments demonstrate the superior performance of WTS across synthetic and real-world datasets, particularly under high-noise conditions.
The source code is available at \href{https://anonymous.4open.science/r/WTS-0F3C}{https://anonymous.4open.science/r/WTS-0F3C}.
\end{abstract}

\keywords{Noisy label learning, Long-tail learning, Pre-trained Model, CLIP}

\section{Introduction}
\label{sec:intro}


With the availability of large-scale public datasets~\cite{ILSVRC15, sun2017JET300, radford2021clip}, significant progress has been made in the field of computer vision~\cite{li2022advances,pang2024heterogeneous} and large models~\cite{radford2021clip,Dosovitskiy21vit}.
However, real-world visual datasets typically exhibit two critical limitations: (1) severe class imbalance, where a small number of head classes dominate the sample distribution while tail classes remain substantially underrepresented~\cite{zhang2023survey}, and (2) pervasive label noise caused by erroneous annotations~\cite{song2019selfie,xiao2015learning}. 
Creating balanced datasets with correctly labeled classes to address these challenges is expensive and unsustainable.
To address these issues, the practical problem of long-tailed noisy label (LTNL) learning~\cite{lu2023tabasco,zhang2023rcal} has been introduced. 

\begin{figure}[!t]
\centering
\begin{subfigure}{0.495\linewidth}
    \includegraphics[width=\linewidth, height=0.6\linewidth]{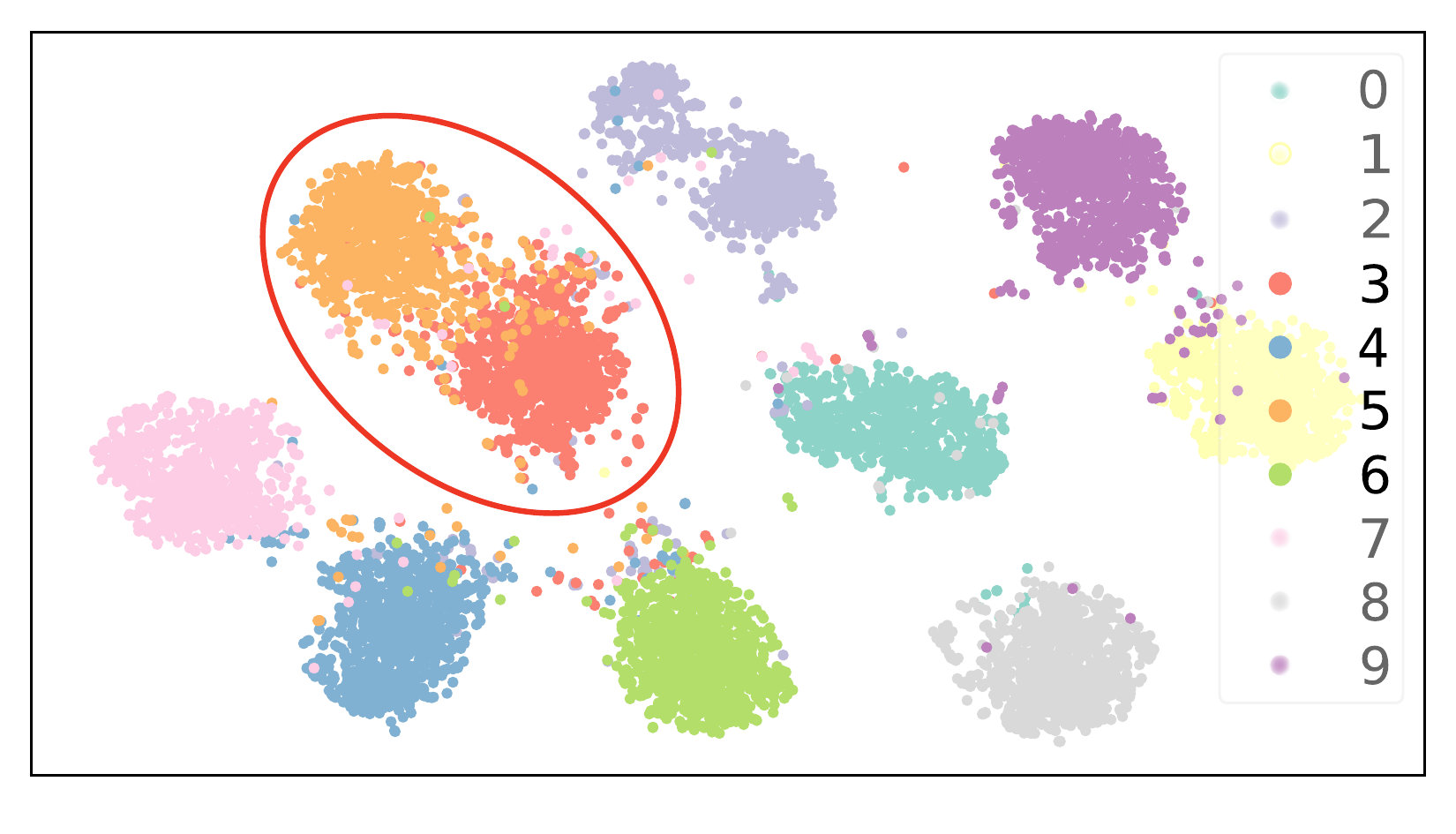}
    \caption{Clean label}
    \label{fig:intro_baseline_a}
\end{subfigure}
\begin{subfigure}{0.495\linewidth}
    \includegraphics[width=\linewidth, height=0.6\linewidth]{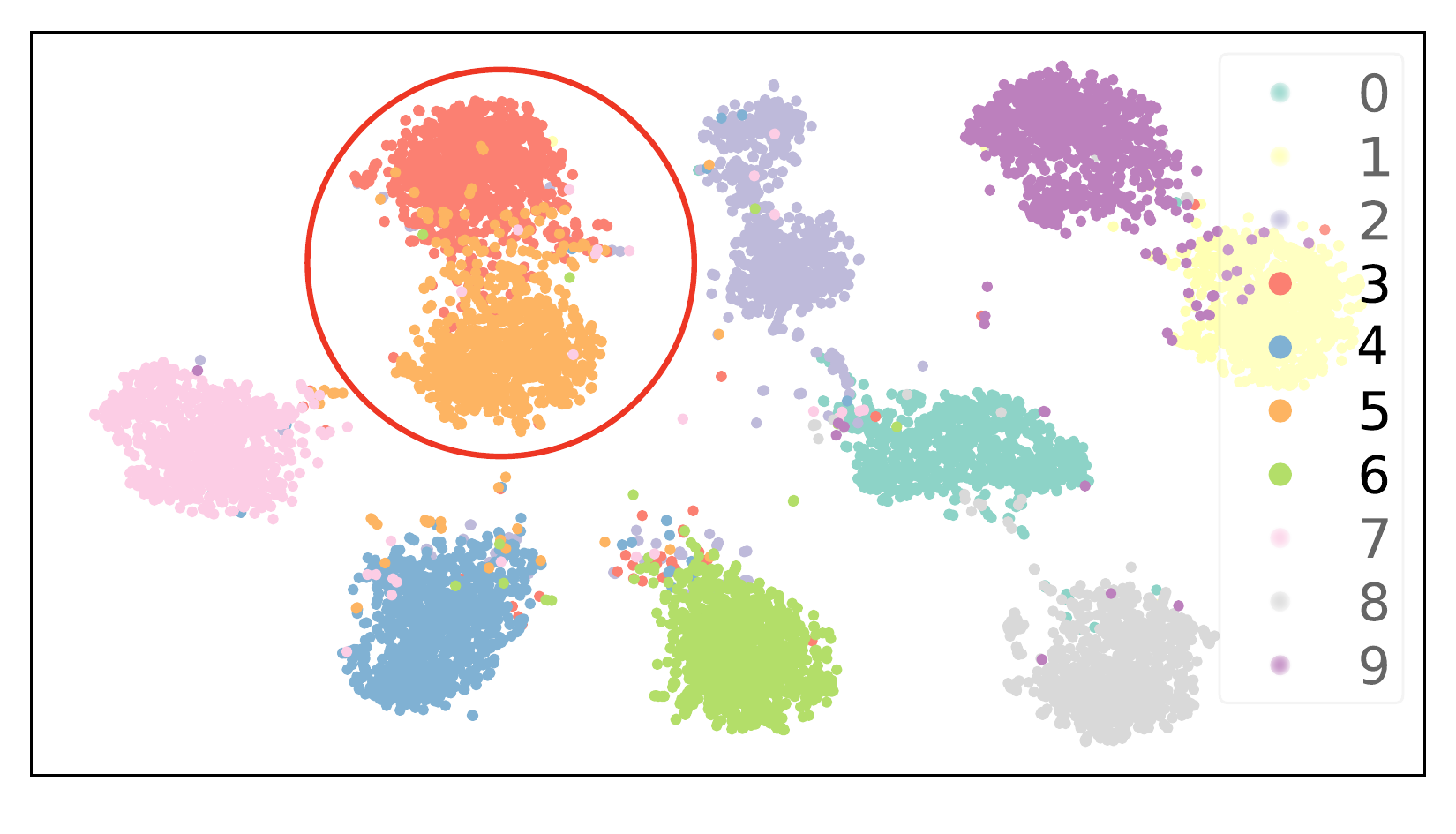}
    \caption{Noisy label (NR=0.1)}
    \label{fig:intro_baseline_b}
\end{subfigure}
\\
\begin{subfigure}{0.495\linewidth}
    \includegraphics[width=\linewidth, height=0.6\linewidth]{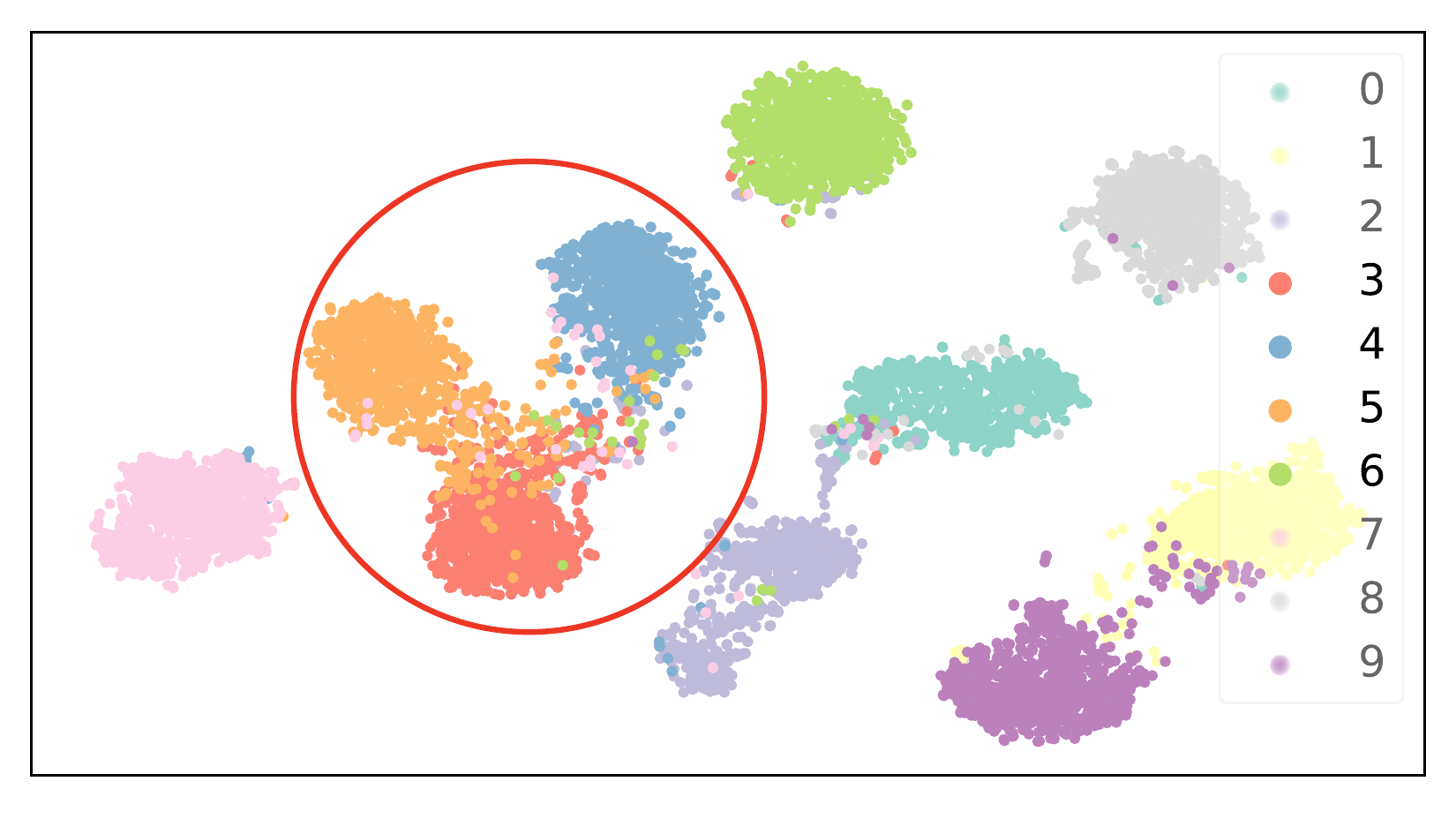}
    \caption{Noisy label (NR=0.5)}
    \label{fig:intro_baseline_c}
\end{subfigure}
\begin{subfigure}{0.495\linewidth}
    \includegraphics[width=\linewidth, height=0.6\linewidth]{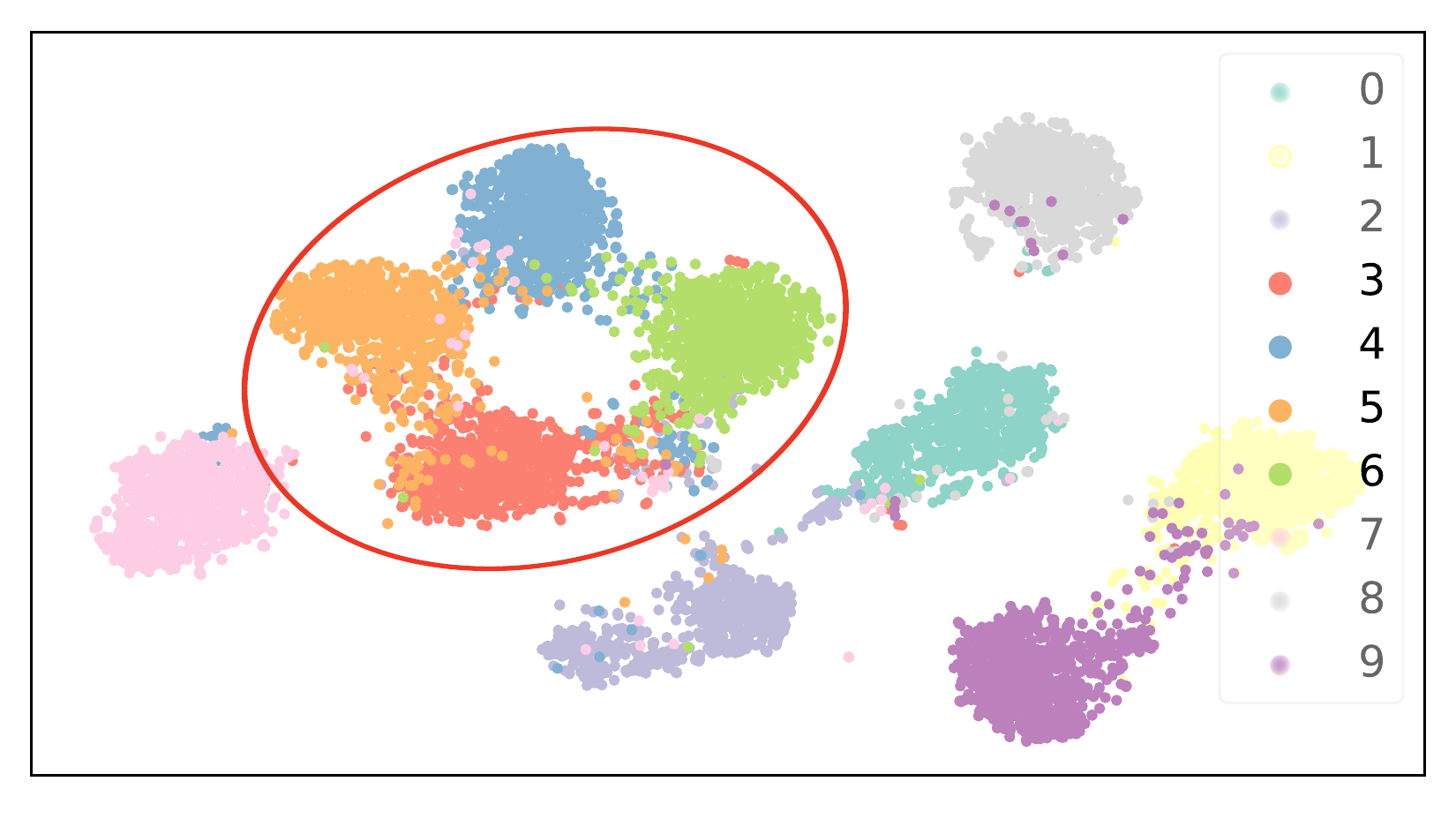}
    \caption{Noisy label (NR=0.6)}
    \label{fig:intro_baseline_d}
\end{subfigure}
\caption{T-SNE visualization of the feature distributions on the test set, obtained by models trained on label-noisy datasets with varying noise rates\protect\footnotemark.
(a) Classes 3 and 5 exhibit overlap.  
(b) Classes 3 and 5 remain partially overlapping. 
(c) Classes 3, 4, and 5 are not entirely separated. 
(d) Classes 3, 4, 5 and 6 are not fully separated. 
}\label{fig:intro_baseline} 
\end{figure}
\footnotetext{\small The training set is CIFAR-10-LTN with asymmetric noise and an imbalance factor of 10. 
The model is Adaptformer~\cite{Chen2022adaptformer}, finetuning on CLIP~\cite{radford2021clip}. 
The loss function employed is logit adjustment~\cite{adjustment21}.
}

Recently, the challenging task of LTNL learning has garnered significant attention.
Broadly, the approaches can be divided into the following categories:
(1) emphasizing the importance of different samples by reweighting or regularization~\cite{Mengye2018Learning, Shu2019mwnet, cao2021heter, wei2022prototypical, jiang2022delving};
(2) selecting clean samples based on carefully designed criteria~\cite{Wei2021robust, xia2022sample, lu2023tabasco, LI2023SFA};
(3) developing improved representation learning methods~\cite{zhou2022prototype, yi2022identifying, zhang2023rcal, li2024extracting}.
The aforementioned methods can effectively enhance the robustness of models on long-tailed noisy labeled data.
However, they overlook the impact of different noise ratios on model training, resulting in suboptimal performance in high-noise scenarios.
Notably, low levels of label noise exert a relatively minimal impact on model performance.
To illustrate this phenomenon, we visualize test set features from models trained on LTNL datasets with varying noise ratios (NRs), as shown in \cref{fig:intro_baseline}.
The feature distributions under low NR closely resemble those of clean labels, as shown in \cref{fig:intro_baseline_a,fig:intro_baseline_b}.
Consequently, long-tailed approaches are less affected in the low noise scenario, as the labels are relatively reliable.
In contrast, when comparing \cref{fig:intro_baseline_c,fig:intro_baseline_d}, we can observe that the feature distributions under high noise ratios differ significantly from those of clean labels. 
Therefore, targeted processing is necessary for high noise ratio scenarios, where unreliable labels constitute one of the primary issues in noisy label learning with long-tailed data.
In such circumstances, noisy labels introduce substantial misleading supervisory signals, making it challenging to effectively distinguish noisy samples from clean ones or improve feature representation.  
This results in accumulated feature learning biases and amplifies the combined challenges of label noise and class imbalance.
To this end, we propose integrating auxiliary linguistic information into supervisory signal calibration, as textual information inherently captures the semantics of the labels themselves and is inherently robust to label noise and data distribution biases in the training set.

Considering that in long-tailed noisy labeled data, the observed labels contain category information but may be inconsistent with the corresponding images, we propose using auxiliary text information from the observed labels to correct these inconsistencies, thereby fully utilizing the label information. 
Specifically, we leverage the text encoder from pre-trained visual-language models (VLMs)~\cite{karpathy2015vse,radford2021clip} to obtain text-based predictions, utilizing this text-image alignment prior to correct label-image inconsistencies.
This text-image alignment prior, serving as a supervisory signal, is not always accurate. 
Therefore, we evaluate the discrepancy between the text-predicted labels from the text encoder and observed labels to decide whether to activate this supervision.
If the predicted labels from the pre-trained text encoder deviate significantly from the observed labels, we consider these text-based predictions to be more informative and incorporate them to guide model training. 
This approach enables the effective application of existing long-tailed learning methods.
Since text-predicted labels generally have lower accuracy than direct fine-tuning of the image encoder, we regard the text encoder as a ``weak teacher'' and refer to our approach as Weak Teacher Supervision (WTS).
Experiments on benchmarks with multiple types of noisy labels and intrinsically long-tailed distributions demonstrate that the proposed WTS improves the performance of the strong student, particularly in scenarios with a high noise ratio.
The main contributions of this paper are summarized as follows:
\begin{itemize}
    \item  
    We empirically demonstrate that even a text encoder from a pre-trained VLM with suboptimal performance can contribute to performance improvements in LTNL learning, and provide an in-depth analysis of the underlying rationale. 
    \item We devise a simple yet effective WTS strategy that integrates seamlessly with various existing methods. 
    It leverages text information to predict image labels and by evaluating the consistency between text-predicted and observed labels, selectively applies supervision to improve label reliability.
    \item Extensive experiments on both simulated and real-world datasets demonstrate the effectiveness of WTS, showing significant performance gains in LTNL learning, especially in challenging high-noise conditions.
\end{itemize}

\section{Related Work}
\label{sec:related}
\subsection{Long-Tail Learning}
Long-tail learning methods typically assume correct labeling within datasets~\cite{cui2019class}.
These methods then apply class-wise operation, generally falling into three main categories~\cite{li2022advances,shi2024LIFT}. 
(1) Input level. 
Data manipulation techniques, such as re-weighting/sampling~\cite{cui2019class} and data augmentation~\cite{CubukED19AutoAugment, CubukED20Randaugment}, are implemented to enhance classification performance.  
(2) Representation level. 
Modifications are made to the model structure to better capture the underlying characteristics of the data. 
\texthl{Decoupling representation~\cite{decouple20, mislas21} and BBN-based methods, where BBN denotes Bilateral-Branch Network~\cite{bbn20, DisAli21}, separate representation learning from classifier training.}
They first extract representations from the original long-tailed dataset, and then retrain the classifier using either class-balanced sampling data~\cite{decouple20} or reverse sampling data~\cite{bbn20}. 
Ensembling learning includes redundant ensembling~\cite{WangXD21RIDE, LiBL2022Trustworthy, liJ2022nested, Cai2021ACE}, which aggregates outputs from separate classifiers or networks within a multi-expert framework, and complementary ensembling~\cite{bbn20, Cui2022reslt}, which involves the statistical selection of different data partitions. 
(3) Output level.
Existing methods enhance model representation and refine the classifier by calibrating model logits based on specific criteria. 
For example, logit adjustment methods~\cite{adjustment21, RenJW2020Balanced} calibrate the predicted output distribution to achieve a balanced distribution. 
Re-margining methods~\cite{Kaidi2019ldam, LiMK2022KPS, adjustment21, LiMK2022GCL} introduce class size-based constants that assign larger margins to tail classes compared to head classes. 


\subsection{Noisy Label Learning}
Noisy label supervision with high-noise datasets critically compromises model recognition performance.
A straightforward and effective approach is to distinguish between clean and noisy samples, with methods such as MentorNet~\cite{2018Mentornet}, Co-teaching~\cite{han2018Co-teachingt} and DivideMix~\cite{li2020Dividemix} treating samples with small training losses as clean samples.
FedFixer~\cite{xinyuan2024fedfixer} introduces the personalized model that collaborates with the global model to effectively select clean and client-specific samples, and NPN~\cite{meng2024npn} directly corrects labels by accumulating model predictions.
Divergence-based approaches rely on margin metrics, such as AUM~\cite{pleiss2020ide} which measures the logit differences between a specified class and the top non-specified class, or distribution similarity employed by methods like Jo-SRC~\cite{Yao2021Jo-src} and UNICON~\cite{Karim2022Unicon} via Jensen-Shannon divergence.
In addition, several methods design noise-robust loss functions to mitigate the impact on noisy data, such as backward and forward loss correction~\cite{Patrini2017mak}, gold loss correction~\cite{Dan2018using}, MW-Net~\cite{Shu2019mwnet} and Dual-T~\cite{Yao2020dual}.
Other effective methods focus on evaluating the noise transfer matrix~\cite{Yao2020dual, Cheng2022class, yexiong2024llcs} or reweighting examples for noisy label learning~\cite{Mengye2018Learning, liu2015class}.

\subsection{Noisy Label Learning on Long-Tailed Data}
Numerous studies have emerged to address the challenges posed by the task of joining noisy labels and unbalanced/long-tailed data.
A common strategy is to distinguish between clean and noisy samples.
For example, CNLCU~\cite{xia2022sample} improves upon the small loss method~\cite{han2018Co-teachingt} by identifying a subset of high-loss samples as clean.
TABASCO~\cite{lu2023tabasco} addresses a complex scenario where noisy labels can cause an intrinsic tail class to be misrepresented as a head class.
To solve this issue, it proposes a bi-dimensional separation metric that effectively adapts to different cases. 
Another promising path is to emphasize the importance of different samples by reweighting or regularization~\cite{Mengye2018Learning, Shu2019mwnet, cao2021heter, wei2022prototypical, jiang2022delving}.
Concurrently, improving representation learning~\cite{zhou2022prototype, yi2022identifying, zhang2023rcal, li2024extracting} leverages intrinsic feature-space structures to mitigate noise propagation.
However, in high-noise environments, the aforementioned methods fall short because noisy labels undermine sample reliability and obscure distinctions between noisy and tail-class samples. 
Moreover, label noise distorts feature space, complicating the utilization of feature-based strategies to address both noise and class imbalance.
\section{Proposed Method: WTS - a Weak yet Effective Teacher}
Noisy labels weaken the reliability of the supervision signal from observed labels, especially at high noise ratios, leading to accumulated biases in feature learning and exacerbating the challenges of label noise and class imbalance.
Fortunately, observed labels still provide category names and counts, making external label support valuable. 
Recent advancements in visual-language models (VLMs)~\cite{karpathy2015vse,radford2021clip} provide a powerful tool for incorporating label information.
To achieve this process, we introduce prediction probabilities from the text encoder of VLM as auxiliary text supervision during model training, referred to as WTS. 
Since WTS may introduce additional errors, we use a switch to determine whether to activate it based on the overlap ratio between observed and text-predicted labels.
The overall structure of the WTS is illustrated in \cref{fig:method}. 

\subsection{Preliminaries}
\noindent\textbf{Problem Definition.}
Consider a training set $\mathcal{D} = \{ (x_i, \hat{y}_i) \}_{i=1}^N$, where each $(x_i, \hat{y}_i)$ pair represents an input and its observed label, and $N$ is the total number of samples in $\mathcal{D}$. 
Suppose $\mathcal{D}$ includes $C$ classes, with class $c$ having $n_c$ training samples. 
Then, the total number of samples is given by $N = \sum_{c=1}^{C} n_c$.
The training set $\mathcal{D}$ exhibits the following properties:
1) Noisy labels. The observed label $\hat{y}_i \in \mathcal{\hat{Y}}$ may be different from the ground truth $y_i \in \mathcal{Y}$. $\mathcal{Y}$ is unavailable.
2) Long-tailed distribution. Without loss of generality, we arrange the classes in descending order by training sample count, so that $n_{1} > n_{2} > \ldots > n_{C}$ with $n_1 \gg n_{C}$.
This learning task is defined as long-tailed noisy label (LTNL) learning~\cite{lu2023tabasco,zhang2023rcal}.
Property 1 results in a distribution derived from $\mathcal{\hat{Y}}$ that is inconsistent with $\mathcal{Y}$. 
Existing long-tailed learning methods typically rely on precise sample counts for each class to effectively adjust logits and/or select suitable structures.
As a result, these methods are inadequate for addressing Property 2, as inaccuracies in category counts can cause over-regularization in certain classes.

\noindent\textbf{Basic Notation.}
In the following sections, scalars are represented by lowercase letters, while vectors are denoted by lowercase boldface letters.
Sets or distributions are represented by uppercase script letters.
The superscripts $I$\footnote{To avoid confusion with the index (subscript $i$), the output of the fine-tuned image encoder is represented by the capital letter $I$.}, $t$ and $o$ are used to differentiate the outputs obtained from the fine-tuned image encoder, the pre-trained text encoder, and the observed labels, respectively. 

\noindent\texthl{\textbf{Logit Adjustment (LA).}
LA~\cite{adjustment21} is a statistically grounded approach that addresses class imbalance by modifying classifier logits based on label frequencies. Formally, given the original logit $z_i$ and the estimated class prior $\pi_i$, the adjusted logit is computed as $\tilde{z}_i = z_i + \log \pi_i$. This adjustment acts as a relative margin that effectively penalizes majority classes during optimization, thereby enforcing consistency with the balanced error rate and improving recognition on tail categories.}

\begin{figure*}[!t]
\centering
\includegraphics[width=1.\textwidth]{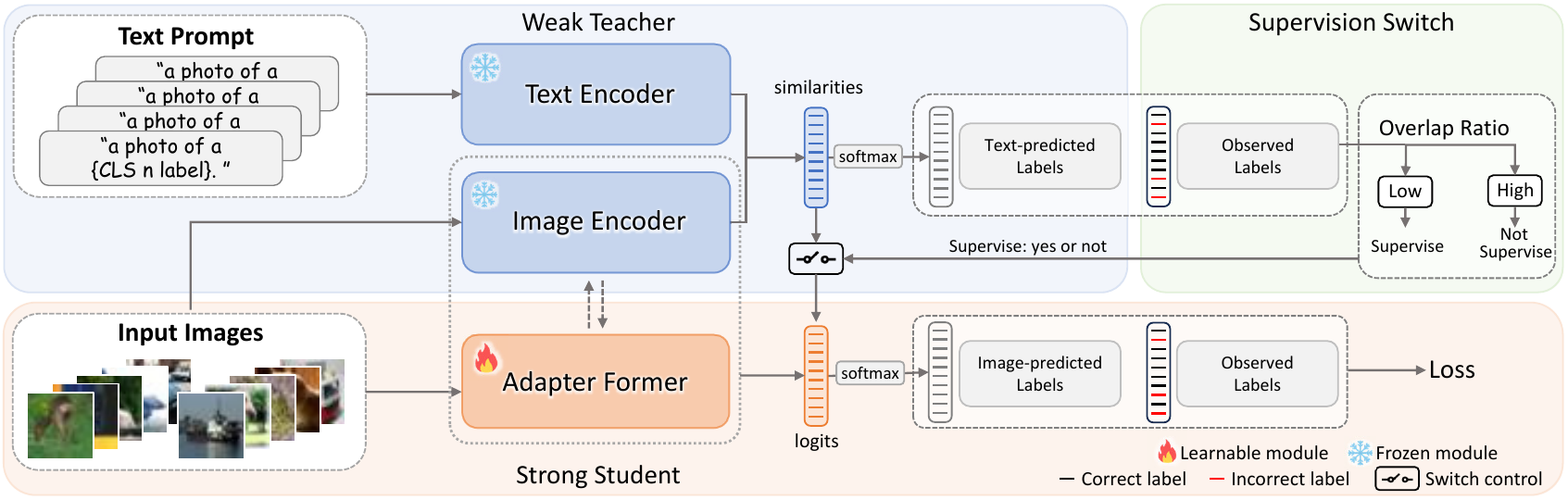}
\caption{Overview of WTS. 
 We leverage the text encoder in pre-trained visual-language models to obtain text-based predictions, using text-image alignment to correct label-image inconsistencies. Since this supervisory signal is not always accurate, we evaluate the discrepancy between the text-predicted and observed labels to determine when to activate it.}
\label{fig:method}
\end{figure*}

\subsection{Weak Teacher Supervision} \label{sec:WTS}

\noindent\textbf{Supervision from Linguistic Information.}
The text and image encoders in a pre-trained VLM can provide text embeddings ($\mathbf{t}_{c}$) of labels for class $c$ and image embeddings ($\mathbf{f}_{i}$) for input images ${x_i}$.
By comparing the similarity between $t_c$ and $f_i$, we can derive the predicted label from the label-based text prompt:
\begin{equation}
   s_{i,c}  =  \frac{{ \mathbf{t}_{c}^T \mathbf{f}_{i}}}{ \|\mathbf{t}_{c}\| \|\mathbf{f}_{i}\|}, \: 
   y_{i} ^t= \mathop{\arg\max}\limits_{c \in [C]} \{s_{i,c}\}^C_{c=1},  
\label{eq:sim}
\end{equation}
where $y_i^t$ is the text-predicted label for input image $x_i$.
Subsequently, a straightforward solution is to integrate the VLM text-predicted label (TL) with the observed label (OL) into the training to provide auxiliary supervision:
\begin{equation}
\label{eq:loss1}
    \mathcal{L} = a \underbrace{\mathcal{L}_{O}\left(x_i,\hat{y}_i\right)}_{\text{OL supervision}} + (1-a) \underbrace{\mathcal{L}_{T}\left(x_i,{y}^t_i\right)}_{\text{TL supervision}},
\end{equation}
where $a$ is a hyper-parameter. 
$\mathcal{L}_{O}$ represents the loss calculated on observed labels, and can employ cross-entropy (CE) loss or existing logit adjustment methods, such as LDAM~\cite{Kaidi2019ldam}, LA~\cite{adjustment21}, and LADE~\cite{Hong2021lade}, for long-tailed learning.
Meanwhile, $\mathcal{L}_{T}$ is the loss based on text-predicted labels.
\cref{eq:loss1} can be utilized to fine-tune a VLM.
However, the text-predicted labels $\mathcal{Y}^t = \{y_i^t\}^N_{i=1}$ also imprecise, for instance, its accuracy on the test set of CIFAR-100 is only 64.4\% (additional results can be found in \cref{sec:cifar_train}), indicating that $\mathcal{L}_{T}\left(x_i,{y}_{T,i}\right)$ introduces another form of label noise.
To address this dilemma, we propose leveraging feature similarity between text and image to provide additional supervision.
Since it can provide not only discrete one-hot optimization objectives but also insights into inter-class relationships.
In detail, we utilize softmax to convert the similarity between the label text features and the input images (as shown in \cref{eq:sim}) into probabilities:
\begin{equation}
   \texthl{p^t(x_i|y=c) = \frac{\exp(s_{i,c})}{\sum_{j=1}^C \exp(s_{i,j})}.}
\end{equation}
For convenience and without loss of generality, for the input $x_i$, we abbreviate this as $p^t_{c}$. 
Similarly, the probability $p^I_{c}$ for the input image can be derived from the similarity between the fine-tuned image features and the classifier weights. 
Following, we can incorporate the text supervision information from the pre-trained VLM into the training process by minimizing the divergence between image and text prediction probability distributions. 
Kullback-Leibler Divergence ($\text{KL}$) is employed in this paper:
\begin{equation}
\label{eq:kl}
    \mathcal{L}_T = \text{KL}(\mathcal{P}^t\| \mathcal{P}^I),
\end{equation}
where $\mathcal{P}^t = \{p^t_{c}\}_{c=1}^C$ and $\mathcal{P}^I = \{p^I_{c}\}_{c=1}^C$ are the probability distributions of text-prediction and fine-tuned image encoder prediction, respectively. 
Since the fine-tuned model has fewer training parameters and has the ability to quickly adapt to new datasets, we treat the image encoder that fine-tuned on LTNL data as a strong student, while the pre-trained VLM serves as a weak teacher and provides weak teacher supervision (WTS).
\texthl{Regarding the temperature scale used in the KL divergence, we implemented it as a learnable parameter rather than a fixed scalar. This allows the model to adaptively adjust the sharpness of the probability distribution during training.}

\noindent\textbf{Supervision Switch Control.}
Since the weak teacher is not always accurate, and as discussed in \cref{sec:intro}, observed labels can be used directly in low-noise scenarios without additional processing.
The challenge, however, lies in the fact that the proportion of noisy labels cannot be determined in advance.
Therefore, an indicator is needed to assess when the teacher provides effective supervision.
\texthl{For a mini-batch of size $B$, the overlap ratio is defined as
$OR = \frac{1}{B} \sum_{i=1}^{B} \mathbf{1} \left( y_i^t = \hat{y}_i \right)$,
where $y_i^t$ and $\hat{y}_i$ denote the text-predicted label and the observed label of the $i$-th sample in the batch, respectively, and $\mathbf{1}(\cdot)$ is the indicator function. 
The overlap ratio $OR$ measures the batch-wise agreement between text-predicted and observed labels and serves as an indicator for activating weak teacher supervision.}
When the overlap ratio $OR$ is high, it indicates that the two types of labels are largely consistent, and the information provided by WTS is limited. 
Therefore, we opt to deactivate it. 
In contrast, when $OR$ is low, the observed labels and the visual-language alignment prior are significantly different, indicating a need for auxiliary supervision.
Then, the supervision switch based on $OR$ can be calculated as:
\begin{equation}
a = 
\begin{cases} 
    1 & \text{if } OR \geq \tau \\
    a \sim \text{Beta}(\alpha, \beta) &  \text{if } OR < \tau 
\end{cases},
\end{equation}
where $\tau$ is the overlap ratio control threshold, which we will empirically analyze in detail in \cref{sec:Fur_Ana}.
\texthl{We used $\alpha = 2.0$ and $\beta = 2.0$ for the Beta distribution.}
We estimate this value in an online manner by calculating the overlap rate between the two types of labels in each batch.
When WTS needs to be turned off, $a=1$, and only $\mathcal{L}_O$ is included in \cref{eq:loss1}. 
Conversely, when the switch control of WTS is turned on, we set $a$ to a random number sampled from the beta distribution.
In this paper, we choose Adapterformer~\cite{Chen2022adaptformer} for VLM fine-tuning.

The algorithm of WTS is summarized in \cref{alg:wts}.

\begin{algorithm}[t]
\caption{Training Algorithm of WTS}
\label{alg:wts}
\begin{algorithmic}[1]
\REQUIRE {Training set $\mathcal{D}$, pre-trained model $\mathcal{M}$;} 
\ENSURE  {Fine-tuned model;}
\STATE   Initialize the fine-tuning module $\phi$;
    \FOR{$k=1$ to $K$}
    \STATE Sample batches data $\mathcal{B}_e \sim \mathcal{D}$; 
    \STATE Compute text-predicted labels $\mathcal{\hat{Y}}^t_{\mathcal{B}_e}$ by~\cref{eq:sim};
    \STATE Compute $OR$ between $\mathcal{\hat{Y}}^t_{\mathcal{B}_e}$ and $\mathcal{\hat{Y}}^o_{\mathcal{B}_e}$;
    \IF{$ OR < \tau $}
        \STATE Compute the TL supervision: $\mathcal{L}_T = \text{KL}(\mathcal{P}^t\| \mathcal{P}^I)$; 
        \STATE Sample $a$ by $a \sim \text{Beta}(\alpha, \beta)$ and compute the final loss by $\mathcal{L}=a \mathcal{L}_{O}+(1-a)\mathcal{L}_{T}$;
    \ELSE
        \STATE Compute the final loss by $\mathcal{L}= \mathcal{L}_{O};$
    \ENDIF
    \STATE Update $\phi$ by $\phi_{k+1} = \phi_{k} - \eta_k \cdot \nabla \mathcal{L}$      
    \ENDFOR
\end{algorithmic}    
\end{algorithm}

\begin{remark}
\label{rem:adv}
The advantages of WTS can be summarized in three main aspects:
 \textbf{(1) Label noise-robust feature calibration.} 
The pre-trained text encoder in VLM, which is unaffected by noisy labels, exclusively focuses on the semantics of the labels and is inherently aligned with visual features, serves as the teacher model to rectify biases in the features obtained by the student model. 
Notably, we observe that fine-tuning with $\mathcal{L}_O$ can sometimes outperform text-prediction, which often has lower accuracy.
Nevertheless, WTS still provides valuable guidance in correcting misalignments introduced by noisy labels.
\textbf{(2) Sample distribution resilient bias corrector.}
WTS predictions are distribution-agnostic, enabling them to effectively mitigate long-tail bias in samples by propagating distribution-independent reference gradients.
This approach helps mitigate potential classification errors that severely affect tail classes and improves the performance of all classes, including both head and tail.
\textbf{(3) Highly efficient training.}
WTS introduces minimal computational overhead, with the primary cost arising from the fine-tuning of the image encoder. 
\end{remark}

\subsection{Effectiveness Analysis of WTS} \label{sec:rationale}
Although the proposed WTS may seem intuitive and straightforward at first glance,  it is built on a solid theoretical foundation. 
In this section, we explore the underlying theoretical rationale behind WTS. 
The efficacy of the proposed method is analyzed from the perspectives of noisy label learning and long-tail learning.

\vspace{0.5em}
\noindent\textbf{Effectiveness on Noisy Label Learning.}
We investigate the impact of WTS on noisy label learning by analyzing how $\mathcal{L}_T$ revise incorrectly observed labels, leading to the following proposition.

\begin{proposition}
\label{prop:NL}
   WTS corrects observed labels based on the predicted probabilities provided by the pre-trained VLM with the ratio $a$. 
\end{proposition}

\begin{proof}
Another form to write \cref{eq:kl} is:
\begin{align}
    \text{KL}(\mathcal{P}^t\| \mathcal{P}^I) &=  -\sum p_c^t \log p_c^I - \left( -\sum p_c^t \log p_c^t \right) \notag \\
    & = -\sum p_c^t \log p_c^I + H(\mathcal{P}^t), \label{eq:p_c}
\end{align}
where $H(\cdot)$ represents cross entropy. 
$\mathcal{P}^t$ is provided by the pre-trained VLM.
Since the VLM parameters are not updated during training, $H(\mathcal{P}^t)$ can be considered a constant throughout the training process.
Therefore, this term can be ignored when optimizing the loss function.
Without loss of generality, in \cref{eq:loss1}, by substituting $\mathcal{L}_T$ with \cref{eq:p_c} and replacing $\mathcal{L}_O$ with the expanded form of the CE-based loss, we can obtain:
\begin{align}
\label{eq:loss}
    \mathcal{L} &= a  \left(-\sum_{c=1}^C  p_c^o \log p_c^I \right) + (1-a) \left(-\sum_{c=1}^C  p_c^t \log p_c^I \right), \notag \\
    & = -\sum_{c=1}^C \left( a\cdot p_c^o+ (1-a)\cdot p_c^t \right)  \cdot \log p_c^I , 
\end{align}
where $p_c^o$ is the one-hot-form probability obtained based on the observed labels.
\end{proof}
\noindent \cref{prop:NL} illustrates that WTS has the following two impacts on noisy labels:  
\begin{itemize}[itemsep=-10pt, topsep=0pt, parsep=0pt, partopsep=0pt]
\item For $\hat{y}_i=y_i$, WTS modifies the observed labels through label smoothing, enabling the preservation of all inter-class relationships, in contrast to relying solely on $\mathcal{Y}^t$; \\
\item For $\hat{y}_i \neq y_i$, WTS prevents over-confidence arising from prediction errors~\cite{li2020Dividemix} by decreasing the probability assigned to the incorrect target class.
\end{itemize}  

\noindent\textbf{Effectiveness on Long-Tail Learning.}
We investigate the influence of WTS on long-tail learning from a gradient-based perspective.
Before proceeding, we introduce a theorem, a used symbol, and a remark in our analysis.

\begin{theorem}
\label{thm:gra}
Let $p$ be the base probability and $q$ be the probability obtained from the softmax function applied to logits $\mathcal{Z}=\{z_i\}^C_{i=1}$ that $q_i = \dfrac{e^{z_i}}{\sum_{j=1}^{C}e^{z_i}}$. 
The cross-entropy loss is  $\ell = - \sum_{i=1}^C p_i \log q_i$.
Then, the derivative of the loss function with respect to the logits $z_i$ is:
\begin{equation}\label{eq:thm}
\frac{\partial \ell}{\partial z_i} = q_i - p_i,
\end{equation}
\texthl{where $p_i$ denotes the target probability of class $i$.}
\end{theorem}

\begin{proof}
For a specific class, we assume, without loss of generality, that class  $c$ is considered.
Then, the derivative on $z_c$ is given by:
\begin{align}
\frac{\partial \ell}{\partial z_c} = & -\left( p_c \frac{ \sum_j e^{z_j} }{e^{z_c} } \cdot \frac{e^{z_c} \sum_j e^{z_j}-e^{z_c}e^{z_c} }{\left(\sum_j e^{z_j}\right)^2}\right. \\
+ & \left. \sum_{i\neq c} p_i \frac{ \sum_j e^{z_i} }{e^{z_c} } \cdot \frac{-e^{z_i} e^{z_c}}{\left(\sum_j e^{z_j}\right)^2 }\right) \notag \\
=&-\left( p_c \frac{\sum_je^{z_j}-e^{z_c}}{\sum_j e^{z_j}} -\sum_{i \neq c}p_i \frac{e^{z_c}}{\sum_j e^{z_j}} \right) \notag \\
=& q_c\sum_{i\neq c}p_i -p_c(1-q_c). \label{eq:grd1}
\end{align}
According to the property that the sum of a probability distribution is 1, we have $\sum_{i\neq c}p_i = 1-p_c$.
Substituting this into \cref{eq:grd1}, we get:
\begin{align}
\frac{\partial \ell}{\partial z_c} & =  q_c(1-p_c)-p_c(1-q_c) \notag \\
& = q_c-p_c. 
\end{align}
By substituting $c$ for $i$, we obtain \cref{eq:thm}.
\end{proof}

\begin{remark}
\label{rem:LA}
The gradient of logit adjustment methods reduces the positive signal contributions from head classes while amplifying those from tail classes.
\end{remark}  
\begin{proof}
We compare the gradient difference $d_g$ for the target class between the LA loss ($\mathcal{L}_{LA}$) and the CE loss ($\mathcal{L}_{CE}$) to analyze how LA adjusts the gradient during model training\footnotemark[6].
According to \cref{thm:gra}, the gradient differences is:
\begin{align}
d_g &= \left(q^{LA}_y -1\right) - \left(q^{CE}_y-1\right)  \notag \\ 
&=\frac{e^{z_y-m_y}}{\sum_j e^{z_j - m_j} } - \frac{e^{z_y}}{\sum_j e^{z_j}}  \notag \\ 
& = \frac{\sum_je^{z_j}\cdot e^{z_y-m_y} - \sum_j e^{z_j - m_j}\cdot e^{z_y}}{\sum e^{z_j - m_j} \sum e^{z_j}}  \notag \\ 
& = \frac{e^{z_y}\left( \sum_j e^{z_j-m_y} - e^{z_j-m_j} \right) }{\sum e^{z_j - m_j} \sum e^{z_j}}. \label{eq:remark1}
\end{align}
The denominator of \cref{eq:remark1} is always positive, thus the sign of $d_g$ is determined entirely by the numerator.
If $y$ belongs to the tail class, in the extreme case where $y$ is the smallest class, $m_y$ exceeds that of the other classes.
Therefore, $d_g<0$, and then $\frac{\partial\mathcal{L}_{LA}}{\partial z_y} < \frac{\partial\mathcal{L}_{CE}}{\partial z_y}$.
Gradient descent updates the parameters by subtracting the gradient to minimize the objective function.
As a result, the positive signal ($1-q_y \geq 0$) for the target class in tail classes is amplified compared to the base loss (CE loss).
For head classes, the opposite holds, that is, the positive signal corresponding to the target class is reduced in tail classes.
\end{proof}
\noindent For the notation, we define the modified probability $p^m_c$ for class $c$ as follows:
\begin{equation}
    p^m_c = a\cdot p^o_c+ (1- a) \cdot p^t_c.
\end{equation}
\cref{thm:gra} gives that the derivatives of $\mathcal{L}_O$ and $\mathcal{L}$ (\cref{eq:loss}) with respect to the logit of the target class are: 
\begin{equation}
\frac{\partial \mathcal{L}_O}{\partial z_y} = p_y^I-1, \: 
\frac{\partial \mathcal{L}}{\partial z_y} = p_y^I-p_y^m. 
\end{equation}

On the one hand, during optimization, the gradient is updated by descending in the opposite direction of the gradient.
Therefore, compared to $\dfrac{\partial \mathcal{L}_O}{\partial z_y}$, $\dfrac{\partial \mathcal{L}}{\partial z_y}$ decreases the positive signal for the target class. 
The extent of this reduction is determined by the text encoder, specifically $p_y^m$, and is independent of the training set distribution.
On the other hand, if $\mathcal{L}_O$ employs the existing logit adjustment method for long-tail learning, according to \cref{rem:LA}, it facilitates automatic gradient balancing.
However, there are errors in the labels. 
Gradients that are incorrectly labeled tail classes will be erroneously amplified, leading to misleading model gradient descent.
WTS introduces text-encoder-derived semantic constraints to attenuate class-specific positive correlations, thereby counteracting error signal amplification while preserving discriminative feature learning.

\begin{table*}[tb]
\small
\centering
\caption{Top-1 acc. (\%) on CIFAR-10/100-LTN with joint noise.
Res32 and Res18 are abbreviations for ResNet-32 and PreAct ResNet18, respectively. 
The best and the second-best results are shown in \underline{\textbf{underline bold}} and \textbf{bold}, respectively.}
\label{tab:cifar_jn}
\vspace{-6pt}
\resizebox{1.\textwidth}{!}{
\setlength\tabcolsep{12pt}  
\renewcommand{\arraystretch}{1.}
\begin{tabular}{l|ccc|ccc|ccc|ccc}
\hlinew{1pt}
Dataset     & \multicolumn{6}{c|}{CIFAR-10-LTN}   & \multicolumn{6}{c}{CIFAR-100-LTN}  \\
\hline
Imbalance Factor  & \multicolumn{3}{c|}{10}  & \multicolumn{3}{c|}{100}  & \multicolumn{3}{c|}{10} & \multicolumn{3}{c}{100}  \\ 
\hline
Noise Ratio  & 0.3 & 0.4 & 0.5  & 0.3 & 0.4 & 0.5 & 0.3 & 0.4  & 0.5 & 0.3 & 0.4 & 0.5\\ 
\hline
CE    & 72.4 & 70.3 & 65.2  &  52.9 &  48.1 &  38.7  & 37.4 & 32.9 & 26.2 & 21.8 & 17.9 & 14.2 \\ 
\cdashline{1-13}
LDAM-DRW\cite{Kaidi2019ldam} &  80.2 & 74.9 & 67.9  & 66.7 & 57.5 & 43.2 &  45.1 & 39.4 & 32.2  & 27.6 & 21.2 & 15.2 \\
NCM~\cite{decouple20} & 74.8 & 68.4 & 64.8 & 60.9 & 55.5 & 42.6 & 41.3 & 35.4 &  29.3 & 24.7 &  21.8 & 16.8 \\
MiSLAS\cite{liu2021improving}  & 83.4 & 76.2 & 72.5  & 67.9 & 62.0 & 54.5 & 50.0 & 46.1 & 40.6 & 32.8 & 27.0 & 21.8 \\ 
\cdashline{1-13}
Co-teaching~\cite{han2018Co-teachingt} & 68.7 & 57.1 & 46.8  & 38.0 & 30.8 & 22.9 & 36.1 & 32.1 & 25.3 & 22.0 & 16.2 & 13.5 \\
CDR~\cite{xia2020robust} & 73.9 & 68.1 & 62.2 & 46.3 & 42.5 & 32.4 & 35.4 & 30.9 & 24.9  & 22.0 & 17.3 & 13.6 \\
Sel-CL+\cite{li2022selective}& 84.4 & 80.4 & 77.3  & 65.7 & 61.4 & 56.2 & 50.9 & 47.6 & 44.9  & 35.1 & 32.0 & 28.6  \\ 
\cdashline{1-13}
RoLT~\cite{Wei2021robust} & 83.5 & 80.9  & 79.0 & 66.5 & 57.9 & 49.0  & 47.4 & 44.6 & 38.6  & 27.6 & 24.7 & 20.1 \\
RoLT-DRW\cite{Wei2021robust}  & 83.6 & 81.4  & 77.1  & 71.1 & 63.6 & 55.1 & 49.3 & 46.3 & 40.9  & 30.2 & 26.6 & 21.1 \\
HAR-DRW\cite{yi2022identifying} & 80.4 & 77.4 & 67.4 &  48.6 & 54.2 & 42.8 & 41.2 & 37.4 & 31.3  & 22.6 & 19.0 & 14.8  \\
RCAL~\cite{zhang2023rcal}  & 84.6 & 83.4 & 80.8  & 72.8 & 69.8 & 65.1  &  51.7 & 48.9 & 44.4  & 36.6 & 33.4 & 30.3 \\
ECBS-Res32~\cite{li2024extracting}  &  87.4 &  85.9 & 84.8 &  76.8 &  75.2 &   73.6  &  53.8 &  52.8 & 51.2  &  38.5 &  37.1 &  35.5 \\
ECBS-Res18~\cite{li2024extracting} &   89.1 &  87.7 & 85.6 &  78.0 &  76.5 &   72.9  &  60.1 &  58.2 & 55.3 &  43.0 &  39.9 &   39.1  \\
\cdashline{1-13}
CLIP (zero-shot)  & 87.2 & 87.2 & 87.2  & 87.2 &  87.2 &  87.2  &64.4 & 64.4 & 64.4 & 64.4 &64.4 & 64.4 \\
CLIP+CE  &  95.1 &  94.8 & 93.9  &  89.9 &  88.2 &   83.5 &  79.7 &  78.6 & 77.5  &  66.4 &  64.8 &   62.1  \\
\rowcolor{mygray}
CLIP+CE+WTS (ours) &  95.2 &  95.1 & 94.5  &  90.1 &  88.9 &   87.8 &  80.8 &  78.7 & 77.8  &  67.7 &  66.1 &   64.9  \\ 
CLIP+LA & \textbf{96.3} & \textbf{96.0}  & \textbf{95.3} &  \textbf{95.2} &  \textbf{94.0} &   \textbf{90.5} &  \textbf{82.0}  & \textbf{80.8}& \textbf{79.7}  &  \textbf{77.5} &  \textbf{76.6}&   \textbf{75.4}  \\
\rowcolor{mygray}
CLIP+LA+WTS (ours) & \underline{\textbf{96.5}} & \underline{\textbf{96.2}} &  \underline{\textbf{95.6}} & \underline{\textbf{95.6}} &  \underline{\textbf{95.2}} &   \underline{\textbf{91.9}}  &  \underline{\textbf{83.2}} & \underline{\textbf{81.2}} & \underline{\textbf{80.4}} & \underline{\textbf{77.8}} & \underline{\textbf{77.1}} &   \underline{\textbf{76.9}} \\      
\hlinew{1pt}
\end{tabular}
}
\end{table*}

\begin{table}[t]
\centering 
\caption{Top-1 acc. (\%) on CIFAR-10/100-LTN with an imbalance factor of 10 under symmetric and asymmetric noise.
}\label{tab:cifar_sn_an}
\
\resizebox{1.\linewidth}{!}{
\renewcommand{\arraystretch}{1.}
\begin{tabular}{l|cc|cc|cc}
\hlinew{1pt}
Dataset   & \multicolumn{2}{c|}{\small CIFAR-10-LTN}  & \multicolumn{2}{c|}{\small CIFAR-100-LTN} & \multicolumn{2}{c}{\small CIFAR-100-LTN}  \\
\hline
Noise Type  &  \multicolumn{4}{c|}{Symmetric} & \multicolumn{2}{c}{Asymmetric} \\
\hline
Noise Ratio & 0.4 & 0.6  & 0.4 & 0.6 & 0.2 & 0.4\\
\hline
CE  & 71.7 & 61.2  & 34.5 & 23.6 &  44.5 & 32.1\\ 
\cdashline{1-7}
LDAM \cite{Kaidi2019ldam} &  70.5 & 62.0  & 31.3 & 23.1 & 40.1 & 33.3\\ 
LA \cite{adjustment21} & 70.6 & 54.9 & 29.1 & 23.2 & 39.3 & 28.5\\ 
IB \cite{park2021influence} & 73.2 & 62.6 & 32.4 & 25.8 & 45.0 & 35.3 \\ 
\cdashline{1-7}
DivdeMix \cite{li2020Dividemix}&  82.7 & 80.2 & 54.7 & 45.0  & 58.1 & 42.0\\ 
UNICON~\cite{Karim2022Unicon}  &  84.3 & 82.3 & 52.3 & 45.9 & 56.0 & 44.7  \\ 
\cdashline{1-7}
MW-Net \cite{Shu2019mwnet} &  70.9 & 59.9 & 32.0 & 21.7 &  42.5 & 30.4 \\ 
RoLT \cite{Wei2021robust} & 81.6 & 76.6 & 42.0 & 32.6 &  48.2 & 39.3\\ 
HAR~\cite{yi2022identifying} & 77.4 & 63.8 & 38.2 & 26.1 & 48.5 & 33.2  \\ 
ULC~\cite{Huang2022ulc} &  84.5 & 83.3 & 54.9 & 44.7 & 54.5  & 43.2 \\
TABASCO~\cite{lu2023tabasco} &  85.5 & 84.8 & 56.5 & 46.0  & 59.4  & 50.5\\
ECBS~\cite{li2024extracting} &  86.4 & 83.9 & 56.7 & 48.1  & 60.5 & 52.1  \\
\cdashline{1-7}
CLIP (zero-shot) & 87.2 & 87.2 & 64.4 & 64.4 & 64.4 & 64.4\\ 
CLIP+CE  &  94.9 &  94.4 & 78.4 & 74.7  & 78.7 &  62.5\\
\rowcolor{mygray}
CLIP+CE+WTS (ours) & 95.2  & 95.0 & 78.9 & 75.9 & 79.1 & \textbf{67.4} \\
CLIP+LA  & \textbf{95.8}  & \textbf{95.2}  & \textbf{80.2} &  \textbf{76.8} & \textbf{79.3} & 67.2  \\
\rowcolor{mygray}
CLIP+LA+WTS (ours) & \underline{\textbf{96.1}} &  \underline{\textbf{95.3}} &  \underline{\textbf{80.7}} &  \underline{\textbf{78.0}} & \underline{\textbf{79.8}} & \underline{\textbf{69.5}}\\ 
\hlinew{1pt}
\end{tabular}
}
\end{table}

\section{Experiment}
\subsection{Basic Settings}

\noindent\textbf{Datasets and Implementation Details.}
We evaluate the proposed WTS on both simulated and real-world noisy long-tailed datasets, following TABASCO~\cite{lu2023tabasco} and RCAL ~\cite{zhang2023rcal}.
Specifically, synthetic scenarios are created based on CIFAR-10/100~\cite{krizhevsky2009learning}, real-world noise is introduced in mini-ImageNet~\cite{jiang2020red} (referred to as red Mini-ImageNet, abbreviated as Img-LTN$^r$), and WebVision-50~\cite{li2017webvision} is used with its inherent noisy labels.
For all constructed datasets, we first subsample a long-tailed version from the original dataset following the exponential decay pattern from prior works~\cite{zhang2023survey}, and then introduce label noise. 
The imbalance factor is defined as the ratio of the largest class size to the smallest.
Three types of noise—joint, symmetric, and asymmetric—are applied to CIFAR datasets.
To distinguish it from the original data, we appended "-LTN" to the constructed long-tailed noisy label dataset.
For model training, we use CLIP~\cite{radford2021clip} ViT-B/16 as the backbone and Adaptformer~\cite{Chen2022adaptformer} as the fine-tuning strategy. 
The optimizer is SGD with an initial learning rate of 0.01, a momentum of 0.9, and a weight decay of $5 \times 10^{-4}$.
The batch size is set to 128, and WTS is trained for 10 epochs across all datasets.
\begin{table*}[!t]
\centering
\caption{\small{Acc. (\%) on CIFAR-10/100-LTN. 
NR is 0.6.
JN, AN, and SN stand for joint, asymmetric, and symmetric noise, respectively.}
} \label{tab:cifar_all_06}
\resizebox{1.\linewidth}{!}{
\setlength\tabcolsep{3pt}
\renewcommand{\arraystretch}{1.}
\begin{tabular}{l|ccc|ccc|ccc|ccc}
\hlinew{1pt}
Dataset     & \multicolumn{6}{c|}{CIFAR-10-LTN}   & \multicolumn{6}{c}{CIFAR-100-LTN}  \\
\hline
Imbalance Factor  & \multicolumn{3}{c|}{10}  & \multicolumn{3}{c|}{100}  & \multicolumn{3}{c|}{10} & \multicolumn{3}{c}{100}  \\ 
\hline
Noise Type  & JN & AN & SN  & JN & AN & SN & JN & AN & SN & JN & AN & SN\\ 
\hline
CLIP+CE   & 91.2 & 19.6 & 94.4 & 71.4 & 22.4 &  87.1 &  74.6 & 21.6 & 74.7  &  58.4 & 23.5 &  62.0\\
\rowcolor{mygray}
CLIP+CE+WTS &  94.1 (\inc{2.9}) &  38.8 & 95.0 (\inc{0.6})  &  83.7 (\inc{12.3}) &  52.5 (\inc{30.1}) &   91.5 (\inc{4.4}) &  75.7 (\inc{1.1}) &  23.4  & 75.9 (\inc{1.2})  &  61.6 (\inc{3.2}) &  29.0  &   65.1 (\inc{3.1})  \\
CLIP+LDAM & 93.0 &1.9 & 94.8 & 75.8 & 11.4 & 87.5 &  74.8 & 7.2& 74.7  & 58.5 &16.5 & 60.0 \\
\rowcolor{mygray}
CLIP+LDAM+WTS  & 94.5 (\inc{1.5}) &  5.1  &  95.0 (\inc{0.2})  &  83.4 (\inc{7.6}) &   26.5   &   90.3 (\inc{2.8})  & 75.2 (\inc{0.4}) & 9.9  & 75.7 (\inc{1.0}) & 60.3 (\inc{1.8}) & 20.3 & 62.6 (\inc{2.6})  \\ 
CLIP+LA &93.9 & 56.7 & 95.2  &  85.9 & 63.1  &  91.7  &  78.4  & 47.1 &  76.8  & 73.7 & 45.7  &  66.8 \\
\rowcolor{mygray}
CLIP+LA+WTS & 94.2 (\inc{0.3}) & 63.8 (\inc{7.1}) & 95.7 (\inc{0.4}) &  88.2 (\inc{2.3}) &  74.7 (\inc{11.6}) & 94.2 (\inc{2.5})  &  79.0 (\inc{0.6}) & 48.8 (\inc{1.7})  & 78.0 (\inc{1.2}) & 74.3 (\inc{0.6}) & 49.5 (\inc{3.8}) & 70.5 (\inc{3.7}) \\ 
\hlinew{1pt}
\end{tabular}
}
\end{table*}

\noindent\textbf{Definition of Noise Types.}
\label{sec:noi_type}
For joint noise, each element $T^{JN}_{ij}$ in the transfer matrix $T^{JN}$ is defined as:
\begin{align}
T^{JN}_{ij} &= P^{JN}(\hat{y}=j\mid y=i)  
= \begin{cases} 
1-\gamma & \text{if } i=j \\ 
\frac{n_j}{N-n_i}\gamma & \text{if } i \neq j
\end{cases},
\end{align}
where $\hat{y}$ denotes the observed label, and $y$ is the ground-truth label. 
$\mathbb{\gamma}$ denotes the noise ratio.
$N$ is the total number of training samples, and $n_i$ is the number of training samples in class $i$.

For symmetric noise, the element $T^{SN}_{ij}$ in the transfer matrix is expressed as:
\begin{equation}
T^{SN}_{ij} = 
\begin{cases} 
\gamma \cdot \frac{1}{C} + (1 - \gamma), & \text{if } i = j \\
\gamma \cdot \frac{1}{C}, & \text{if } i \neq j 
\end{cases},
\end{equation}
where $C$ represents the total number of classes.

For asymmetric noise, the element $T^{AN}_{ij}$ in its  transfer matrix is given by:
\begin{equation}
T^{AN}_{ij} = 
\begin{cases} 
1-\gamma, & \text{if } i = j \\
\gamma \cdot P(i \rightarrow j), & \text{if } i \neq j
\end{cases},
\end{equation}
where $P(i \rightarrow j)$ is the probability of a sample with true label $i$ being mislabeled as $j$ and satisfies $\sum_{j=1}^C P(i \rightarrow j) = 1$. 
Similar to previous work~\cite{Shu2019mwnet, lu2023tabasco}, we adopt $P(i \rightarrow j)$, where only one element is 1 while all others are 0 in our experiments.

\subsection{Comparison Results}

\noindent \textbf{Comparison Methods.}
We compare our method with the following three types of approaches: 
(1) \textit{Long-tail (LT) learning methods}: LDAM~\cite{Kaidi2019ldam}, NCM~\cite{decouple20},  MiSLAS~\cite{liu2021improving}, logit adjustment (LA)~\cite{adjustment21}, and influence-balanced loss (IB)~\cite{park2021influence}. 
(2) \textit{Label-noise (LN) learning methods }: Co-teaching (CT)~\cite{han2018Co-teachingt}, CDR~\cite{xia2020robust}, Sel-CL~\cite{li2022selective} DivideMix~\cite{li2020Dividemix}, and UNICON~\cite{Karim2022Unicon}. 
(3) \textit{Long-tailed noisy label (LTNL) learning}: MW-Net~\cite{Shu2019mwnet}, ROLT~\cite{Wei2021robust}, HAR~\cite{yi2022identifying}, ULC~\cite{Huang2022ulc},  TABASCO~\cite{lu2023tabasco}, RCAL~\cite{zhang2023rcal} and ECBS~\cite{li2024extracting}.

\noindent\textbf{Results on CIFAR-10/100-LTN.}
\cref{tab:cifar_jn,tab:cifar_sn_an} compare joint versus symmetric/asymmetric noise results on CIFAR-10/100-LTN.
We observe that directly applying the LT learning method can achieve improvement to a certain extent. 
The logit adjustment-based method performs slightly better. 
NCM requires resampling the data distribution based on class labels, but the presence of noisy labels leads to unreasonable resampling, limiting its performance improvement.
Under joint noise, LN learning methods, such as Sel-CL+~\cite{li2022selective}, outperform LT learning methods. 
For symmetric and asymmetric noise, recently proposed noise learning techniques can achieve satisfactory performance. 
However, these two kinds of methods are less effective when the noise ratio is high.
For example, under joint noise, when the imbalance factor (IF) of CIFAR-100-LTN is 100 and the noise ratio (NR) is 0.5, MiSLAS and Sel-CL+ achieve accuracies of 21.8\% and 28.6\%, respectively. 
While these are significantly higher than CE (14.2\%), they still fall short of meeting practical requirements for usage.

\cref{tab:cifar_all_06} presents the performance comparison of various methods on CIFAR-10/100-LTN under an NR of 0.6, representing an extremely high-noise regime.
In this challenging setting, observed labels are highly unreliable, leading to poor performance from all methods under asymmetric noise conditions. 
The experimental results demonstrate that incorporating WTS can significantly enhance the performance of all methods, especially under asymmetric and symmetric noise.
For example, WTS improves the CLIP+LA method by more than 10\%, achieving 74.7\% compared to 63.1\% on the CIFAR-10-LTN dataset with an imbalance factor of 100.

Recent proposed LTNL learning methods, including RCAL~\cite{zhang2023rcal}, TABASCO~\cite{lu2023tabasco} and ECBS~\cite{li2024extracting}, have shown enhanced robustness across various noise ratios. 
However, there remains room for further improvement, particularly on more challenging datasets. 
For example, on CIFAR-100 with an IF of 10 and NR of 0.4, ECBS achieves accuracies of 58.2\%, 56.7\%, and 52.1\% under three types of label noise, respectively.
In comparison, the proposed WTS achieves 81.2\%, 80.7\%, and 69.5\%, highlighting the effectiveness of the CLIP-introduced prior and the text-based knowledge in WTS.
These results demonstrate the generalization capability of WTS across different noise types and its robustness under high-noise conditions.

\begin{figure}[t]    
\centering
\begin{minipage}[c]{0.475\linewidth}
\centering 
\captionof{table}{Top-1 acc. (\%) on Img-LTN$^r$ with NR of 0.4.
}
\label{tab:red_img}
\resizebox{1.\linewidth}{!}{
\setlength{\tabcolsep}{4pt}
\begin{tabular}{l|cc}
\hlinew{1.pt}
Imbalance Factor  & 10 & 100  \\ 
\hline
CE  & 31.5  &  31.5 \\
\cdashline{1-3} 
LDAM~\cite{Kaidi2019ldam}  &   23.5 & 15.6  \\ 
LA~\cite{adjustment21}  & 25.9 & 9.6 \\ 
IB~\cite{park2021influence}   & 22.1 & 16.3    \\ 
\cdashline{1-3} 
DivdeMix~\cite{li2020Dividemix}  & 49.0 & 34.7 \\ 
UNICON \cite{Karim2022Unicon}  & 41.6 & 31.1\\ 
\cdashline{1-3} 
MW-Net \cite{Shu2019mwnet}  &   40.3 & 31.1   \\ 
RoLT \cite{Wei2021robust}   &  24.2  & 16.9 \\ 
HAR~\cite{yi2022identifying}  &  38.7  &31.3   \\ 
ULC~\cite{Huang2022ulc}   & 47.1 & 34.8  \\
TABASCO~\cite{lu2023tabasco}   &  49.7 & 37.1 \\
ECBS~\cite{li2024extracting}   &  50.8 & 36.9  \\
\cdashline{1-3} 
CLIP (zero shot)  & 77.1  &  77.1 \\
CLIP+CE  & 82.9 & 80.5  \\
\rowcolor{mygray}
CLIP+CE+WTS   & \underline{\textbf{83.3}} & \underline{\textbf{81.3}} \\
CLIP+LA   & 81.9 & 79.5 \\
\rowcolor{mygray}
CLIP+LA+WTS & \textbf{83.1} &  \textbf{80.9}  \\ 
\hlinew{1pt}
\end{tabular}
}
\end{minipage}
\hspace{1pt}
\begin{minipage}[c]{0.5\linewidth}
\centering
\captionof{table}{Top-1 acc. (\%) on WebVision-50.} 
\label{tab:webvision_img}
\resizebox{1.\linewidth}{!}{
\renewcommand{\arraystretch}{1.15}
\setlength\tabcolsep{1pt}
\begin{tabular}{l|cc}
\hlinew{1.1pt}
Train   & \multicolumn{2}{c}{WebVision-50}   \\
\hline
Test    &WV50\footnotemark[3] & IMG12\footnotemark[3] \\
\hline
CE      & 62.5  & 58.5 \\ 
\cdashline{1-3} 
CT~\cite{han2018Co-teachingt}  & 63.6 & 61.5 \\
MentorNet~\cite{2018Mentornet}  & 63.0  & 57.8  \\
ELR+~\cite{liu2020elr} & 77.8  &  70.3 \\
MoPro~\cite{junnan2021mopro}  & 77.6  & 76.3 \\
NGC~\cite{WU2021ngc}  &  79.2 & 74.4 \\ 
Sel-CL+~\cite{li2022selective}   &  80.0 & 76.8  \\ 
RCAL+~\cite{zhang2023rcal}  & 79.6 & 76.3  \\
ECBS~\cite{li2024extracting}  &  80.0 & 76.1   \\  
\cdashline{1-3}
CLIP (zero-shot)  & 74.5  & 78.0 \\ 
CLIP+CE   & 83.4 & 83.8   \\
\rowcolor{mygray}
CLIP+CE+WTS & 83.5    &  83.6  \\ 
CLIP+LA   & \textbf{85.2}   & \textbf{84.1} \\
\rowcolor{mygray}
CLIP+LA+WTS & \underline{\textbf{85.2}}  & \underline{\textbf{84.2} }  \\
\hlinew{1.1pt}
\end{tabular}
}
\end{minipage}
\vspace{-1.5em}
\end{figure}

\begin{table*}[t]
\centering
\caption{Top-1 accuracy (\%) on Red mini-ImageNet dataset with real-world noise. }
\resizebox{1\linewidth}{!}{
\setlength\tabcolsep{2.5pt}
\renewcommand{\arraystretch}{1.2}
\begin{tabular}{l|cccccc|cccccc}
\hlinew{1pt}
Imbalance Factor  & \multicolumn{6}{c|}{10}   & \multicolumn{6}{c}{100}  \\ 
\hline
Noise Ratio  & 0.1 & 0.2 & 0.3 & 0.4 & 0.5 & \multicolumn{1}{c|}{0.6}  & 0.1 & 0.2 & 0.3 & 0.4 & 0.5 & 0.6  \\
\hline
CLIP (zero-shot) & 77.1 & 77.1 & 77.1 & 77.1 & 77.1 & 77.1  & 77.1 & 77.1 & 77.1 & 77.1 &  77.1 & 77.1  \\
\cdashline{1-13}
CLIP+CE & 85.7 &  84.4 & 84.1 &  82.9 & 80.9 & 79.6  &   \textbf{84.0} &  \textbf{82.6} & 81.3 &  80.5 &   79.7  & 78.4   \\
\rowcolor{mygray}
CLIP+CE+WTS & \underline{\textbf{86.1}} (\inc{0.4})   &  \textbf{85.3} (\inc{0.9})  & \textbf{85.1} (\inc{1.0})  & \underline{\textbf{83.3}} (\inc{0.4}) & 82.8 (\inc{1.9}) & \underline{\textbf{82.8}} (\inc{3.2}) & \underline{\textbf{84.2}} (\inc{0.2}) & \underline{\textbf{82.7}} (\inc{0.1}) & \underline{\textbf{82.6}} (\inc{1.3}) & \underline{\textbf{81.3}} (\inc{0.8}) &\textbf{81.0} (\inc{1.3}) & \underline{\textbf{80.6}} (\inc{2.2})  \\
\cdashline{1-13}
CLIP+LDAM & 84.8 &  83.4 &  82.7 &  81.2  & 79.8 & 77.9 & 80.2  &  78.3 &  77.8 & 77.3 &  75.5  & 74.2  \\
\rowcolor{mygray}
CLIP+LDAM+WTS &  84.8 (0.0) &  84.1 (\inc{1.7})&  82.8 (\inc{0.1}) & 82.4 (\inc{1.2}) & 81.5 (\inc{1.7}) & 80.4 (\inc{2.5}) & 80.2 (0.0) &  78.9  (\inc{0.6})&  78.3 (\inc{0.5})&  77.6  (\inc{0.3})& 77.8  (\inc{2.3})  & 76.1  (\inc{1.9}) \\
CLIP+LADE & 84.4 &  83.6 & 83.1 & 81.9  & 81.1 & 79.7 & 79.9 &  80.5 &  79.1 & 78.9 &  78.4  &77.8 \\
\rowcolor{mygray}
CLIP+LADE+WTS  & 84.8 (\inc{0.4}) & 84.2 (\inc{0.6}) & 83.4 (\inc{0.3})& 83.0 (\inc{1.1}) & \textbf{83.5} (\inc{2.4}) & 82.2 (\inc{2.5}) & 80.4 (\inc{0.5}) & 81.5 (\inc{1.0}) & 79.6 (\inc{0.5}) &  79.7 (\inc{0.8}) &  80.3 (\inc{1.9})  & \textbf{79.3} (\inc{1.5})  \\
CLIP+LA &  85.1 &  84.5 &  83.1 &  81.9  & 81.1 &79.0  & 81.3 &  81.0 &  79.9 &  79.5 &  78.9  & 77.6    \\
\rowcolor{mygray}
CLIP+LA+WTS  & \textbf{86.0} (\inc{0.9}) & \underline{\textbf{85.9}} (\inc{1.4}) & \underline{\textbf{85.7}} (\inc{2.6}) & \textbf{83.1} (\inc{1.2}) & \underline{\textbf{83.6}} (\inc{2.5}) & \textbf{82.5} (\inc{3.5}) & 83.4 (\inc{2.1}) & 82.6 (\inc{1.6}) &  \textbf{81.5} (\inc{1.6}) &  \textbf{80.9} (\inc{1.4}) &  \underline{\textbf{81.4}} (\inc{2.5})  & \underline{\textbf{80.6}} (\inc{3.0}) \\
\hlinew{1pt}
\end{tabular}
}
\label{tab:img}
\end{table*}

\begin{figure*}[t]
\centering 
\begin{subfigure}{0.33\linewidth}
    \includegraphics[width=1\linewidth,height=0.5\linewidth]{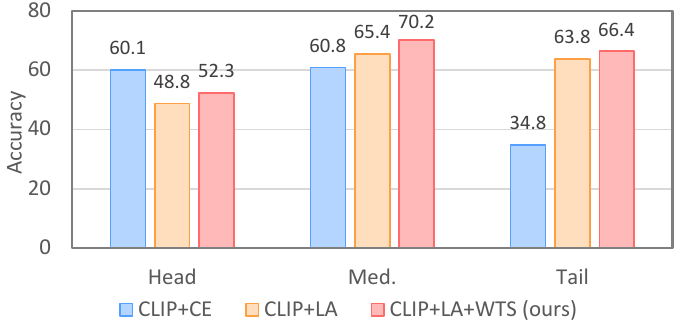} 
    \caption{NR = 0.4} 
    \label{fig:acc40} 
\end{subfigure}
\begin{subfigure}{0.33\linewidth}
    \includegraphics[width=1\linewidth,height=0.5\linewidth]{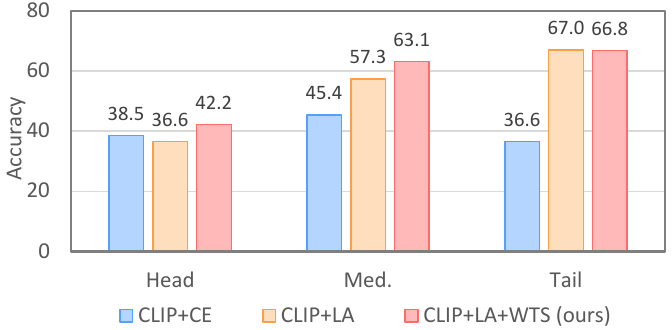} 
    \caption{NR = 0.5} 
    \label{fig:acc50} 
\end{subfigure}
\begin{subfigure}{0.33\linewidth}
    \includegraphics[width=1\linewidth,height=0.5\linewidth]{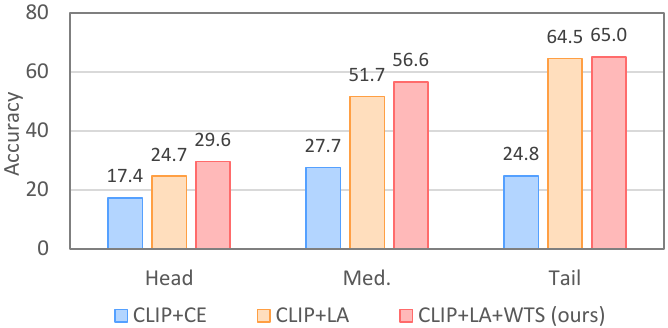} 
    \caption{NR = 0.6} 
    \label{fig:acc60} 
\end{subfigure}
\caption{Accuracy of different class types. (CIFAR100-LTN with IR of 100 and asymmetric noise)} 
\label{fig:acc_cls_types}
\end{figure*}

\begin{figure}[t]
\centering
    \includegraphics[width=1.\linewidth]{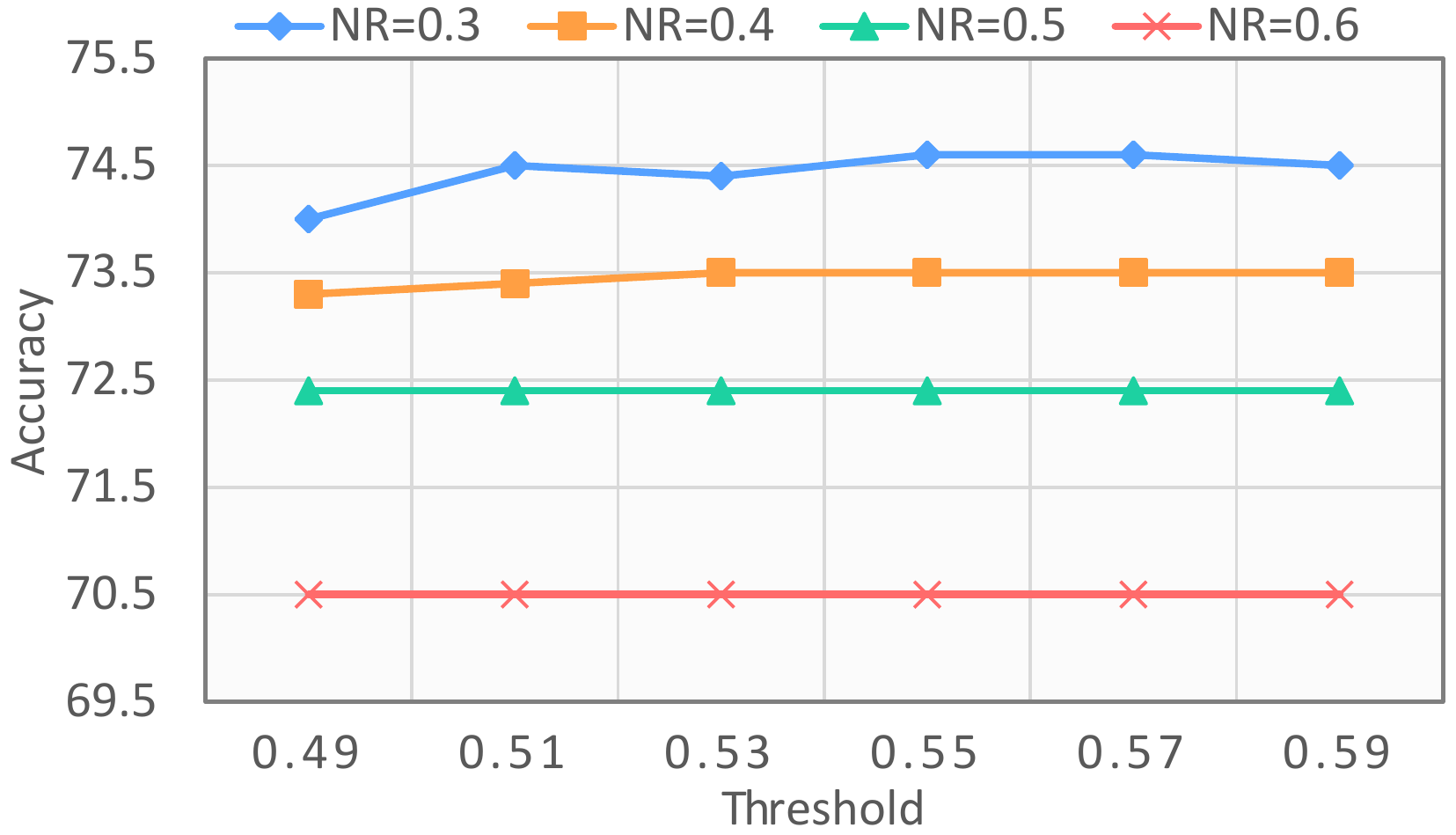}  
    \caption{Ablation of $\tau$ in supervision switch. 
   The dataset is CIFAR-100-LTN with IR=100 and symmetric noise.
    } \label{fig:abl_tao} 
    \vspace{-0.5em} 
\end{figure}

\noindent \textbf{Results on Real-World Datasets.}
\cref{tab:red_img} reports the performance on the test set of Img-LTN$^r$.
It can be observed that with the noise ratio reaching 0.4, directly applying LT learning methods can lead to adverse effects.
Similar to the results on the CIFAR-10/100-LTN datasets, the LN learning method shows effectiveness on Img-LTN$^r$ with a low imbalance ratio. 
In contrast, the LTNL learning method demonstrates a more significant improvement. 
For instance, when IF is 100, TABASCO~\cite{lu2023tabasco} achieves a top-1 classification accuracy of 37.1\%, compared to 34.7\% achieved by DivideMix~\cite{li2020Dividemix}.
In comparison, WTS exceeds 80\%.
Under real-world noisy label conditions, LA also negatively impacts the CLIP fine-tuned model, with CLIP+LA reducing CE from 82.9\% to 81.9\% at an imbalance ratio of 10 for example. 
In contrast, WTS enables the effective use of LA, further boosting performance to 83.1\%.

WebVision-50 is derived from real-world datasets with NR of 0.05, out-of-distribution ratio 0.24~\cite{Albert2022oodLN} and IF of 6.78. 
Since the NR is low, LA can be applied directly, and the correction effect of WTS is less evident. 
However, WTS does lead to a slight improvement in cross-dataset test results, demonstrating an enhancement in the generalization ability of the models.

\footnotetext[3]{WV50 and IMG12 are abbreviated for WebVision-50 and ILSVRC12~\cite{alex2012imagenet}, respectively.}

\subsection{Further Analysis}
\label{sec:Fur_Ana}



\noindent\textbf{Impact of WTS on Different Classes.}
We conduct experiments to demonstrate that WTS can enhance the performance of all classes. 
\cref{fig:acc_cls_types} shows the top-1 classification accuracy across different class types.
As shown in the figures, compared to the basic CE loss, LA improves tail-class performance at noise ratios of 0.4 and 0.5, with a slight trade-off in head-class performance. 
Building on LA, WTS further improves the performance of all classes.

\noindent\textbf{The Influence of Overlap Ratio Control Threshold $\tau$.}
The supervision switch control in \cref{sec:WTS} relies on a hyper-parameter $\tau$.
We conduct an ablation study on the parameter $\tau$ to examine its impact on model training.
\cref{fig:abl_tao} shows the results.
When the noise ratio is low (e.g., 0.3 or 0.4), the choice of parameter $\tau$ influences model performance. 
In such cases, the reliability of observed labels is relatively higher
It is crucial to assess whether supervision from the weak teacher may teacher introduce uncertainty, potentially impacting training. 
Therefore, determining whether to trust the label $\mathcal{Y}^t$, which depends on $\tau$, becomes essential.  
\cref{fig:abl_tao} shows minimal fluctuation, demonstrating that WTS remains stable and informative across different values of $\tau$.
When NR is high, the reliability of observed labels decreases, making the supervision signal from the weak teacher more critical.
The observed label set $\mathcal{Y}^o$ deviates significantly from $\mathcal{Y}^t$, further emphasizing the necessity of guidance from the weak teacher
Consequently, during training, the control switch remains active, ensuring that model performance is largely unaffected by $\tau$.

\begin{table}[!t]
\centering 
\caption{Acc. (\%) of CLIP (zero-shot) on the training set of CIFAR-100-NLT.
}\label{tab:cifar100-train}
\resizebox{0.9\linewidth}{!}{
\setlength\tabcolsep{9pt}
\renewcommand{\arraystretch}{1.}
\begin{tabular}{c|cccc}
\hlinew{0.8pt}
\hline
Imbalance Factor  & Head & Med. & Tail & All \\
\hline
10 & 66.5 & 64.1  & 66.1& 65.5 \\
\hline
100 & 66.8 & 64.0  & 60.6& 64.0 \\
\hlinew{1pt}
\end{tabular}
}
\end{table}

\noindent\textbf{CLIP performance on training set.}
\label{sec:cifar_train}
\cref{tab:cifar100-train} presents the performance of CLIP on the training set of CIFAR-100-LTN. 
Since the predictions obtained by CLIP are independent of class labels, they remain unaffected by label noise and class distribution biases.
However, the predicted labels generated by the text encoder of CLIP are not sufficiently accurate.
We therefore characterize the textual supervision from CLIP as a weak supervisory signal that, while imperfect, provides valuable guidance for model training under noisy conditions.

\section{Concluding Remarks}
In this work, we proposed a label calibration method, WTS, to tackle the compounded challenges of noisy labels and long-tailed distributions in real-world data. 
WTS leverages auxiliary language information from pre-trained visual-language models to correct label misalignment.
By calibrating the supervisory signal, WTS enables effective feature learning and ensures that valuable category information is preserved, even in high-noise scenarios. 
This approach shows significant improvements in model performance across various benchmarks, particularly under challenging noise conditions.
Despite WTS being effective in most scenarios, its supervision activation relies on an empirically chosen hyperparameter. 
In simple noise environments, this dependency may occasionally lead WTS to provide misleading signals, potentially impacting model performance.
Our future work will focus on developing a more reasonable parameter selection to overcome this limitation.


\section*{Acknowledgements}
This work was supported in parts by NSFC (62306181), Guangdong Basic and Applied Basic Research Foundation (2024A1515010163), Shenzhen Science and Technology Program (RCBS20231211090659101, KJZD2024090 3100022028), National Key Lab of Radar Signal Processing (JKW202403), and Scientific Development Funds from Shenzhen University.


{\small
\bibliographystyle{cvm}
\bibliography{cvmbib}
}

\end{document}